\documentclass{article} 
\usepackage{times}
\usepackage[accepted]{icml2018}
\usepackage{booktabs}       
\usepackage{hyperref}
\usepackage{url}
\usepackage{xspace} 
\usepackage{amsfonts}
\usepackage{amsthm}
\usepackage{amsmath}
\usepackage{amssymb}
\usepackage{pdfpages}
\usepackage{graphicx}
\usepackage{subcaption}

\usepackage{tikz}
\usetikzlibrary{circuits.logic.US}
\usetikzlibrary{shapes.misc}
\usetikzlibrary{shapes.arrows}
\usetikzlibrary{arrows,shapes.geometric}

\tikzstyle{bnarrow}=[
    decoration={markings,mark=at position 1 with {\arrow[scale=1.5]{>}}},
    postaction={decorate},
    shorten >=0.7pt,
    >=latex,
    line width=0.3
]
\tikzstyle{bayesnet}=[
  >=latex, thick, auto
]
\tikzstyle{bnnode}=[
  draw,ellipse,minimum size=7mm,inner sep=1pt,font=\small
]

\makeatletter
\def\thm@space@setup{%
  \thm@preskip=\topsep
  \thm@postskip=0.1\thm@preskip 
}
\makeatother
\newtheorem{theorem}{Theorem}

\newtheorem{lemma}[theorem]{Lemma}
\newtheorem{proposition}[theorem]{Proposition}

\theoremstyle{definition}
\newtheorem{definition}{Definition}
\newtheorem{axiom}{Axiom}

\newcommand{\ra}[1]{\renewcommand{\arraystretch}{#1}}

\newcommand{\rvar}[1]{\ensuremath{\mathit{#1}}\xspace}
\newcommand{\Xr}{\rvar{X}}

\newcommand{\rvars}[1]{\ensuremath{\mathbf{#1}}\xspace}
\newcommand{\Xs}{\rvars{X}}
\newcommand{\Ys}{\rvars{Y}}

\newcommand{\jstate}[1]{\ensuremath{\mathbf{#1}}\xspace}
\newcommand{\xs}{\jstate{x}}
\newcommand{\ys}{\jstate{y}}

\newcommand{\true}{\mathit{true}}

\newcommand{\vect}[1]{\ensuremath{\mathbf{\mathsf{#1}}}\xspace}
\newcommand{\pv}{\vect{p}}
\newcommand{\qv}{\vect{q}}

\newcommand{\sloss}{\operatorname{L^s}}

\icmltitlerunning{A Semantic Loss Function for Deep Learning with Symbolic Knowledge}

\author{Jingyi Xu, Zilu Zhang, Tal Friedman, Yitao Liang \& Guy Van den Broeck\\
Computer Science Department\\
University of California, Los Angeles\\ 
Los Angeles, CA, USA\\
\texttt{\small jixu@cs.ucla.edu,zhangzilu@pku.edu.cn,\{tal,yliang,guyvdb\}@cs.ucla.edu} 
}

\newcommand{\toycircuit}{
\centering
\begin{tikzpicture}[circuit logic US]
\node (x1) at (0,0) {$x_1$};
\node (nx2) at (1,0) {$\neg x_2$};
\node (nx3) at (2,0) {$\neg x_3$};

\node (nx1) at (3,0) {$\neg x_1$};
\node (x2) at (4,0) {$x_2$};

\node (x3) at (5,0) {$x_3$};

\node (and1) [and gate, inputs=nnn, rotate=90, scale =0.65] at ($(nx2) + (0,1.25)$) {};
\draw (and1.input 1) -- (x1);
\draw (and1.input 2) -- (nx2);
\draw (and1.input 3) -- (nx3);

\node (and2) [and gate, inputs=nnn, rotate=90, scale = 0.65] at ($(and1) + (1.5,0)$) {};
\draw (and2.input 1) -- (nx3);
\draw (and2.input 2) -- (nx1);
\draw (and2.input 3) -- (x2);

\node (and3) [and gate, inputs=nnn, rotate=90, scale=0.65] at ($(and2) + (1.5,0)$) {};
\node (dummy1) at ($(nx2) + (-0.03,0.13)$) {};
\node (dummy2) at ($(nx1) + (-0.1,0.13)$) {};
\draw (and3.input 1) -- (dummy1);
\draw(and3.input 2) -- (dummy2);
\draw (and3.input 3) -- (x3);

\node (or) [or gate, inputs=nnn, rotate=90, scale=0.65] at ($(and2) + (0,1.25)$) {};
\draw (or.input 1) -- (and1.output);
\draw (or.input 2) -- (and2.output);
\draw (or.input 3) -- (and3.output);

\end{tikzpicture}
}

\newcommand{\circuitComputation}{
\centering
\begin{tikzpicture}[circuit logic US]
\node (x1) at (0,0) {$\Pr(x_1)$};
\node (nx2) at (1.6,0) {$\Pr(\neg x_2)$};
\node (nx3) at (3.2,0) {$\Pr(\neg x_3)$};

\node (nx1) at (4.8,0) {$\Pr(\neg x_1)$};
\node (x2) at (6.4,0) {$\Pr(x_2)$};

\node (x3) at (8,0) {$\Pr(x_3)$};

\node (and1) at ($(nx2) + (0,1.5)$) {$\times$};
\draw (and1) -- (x1);
\draw (and1) -- (nx2);
\draw (and1) -- (nx3);

\node (and2) at ($(and1) + (1.6,0)$) {$\times$};
\draw (and2) -- (nx3);
\draw (and2) -- (nx1);
\draw (and2) -- (x2);

\node (and3) at ($(and2) + (1.6,0)$) {$\times$};
\node (dummy1) at ($(nx2) + (0.1,0.25)$) {};
\node (dummy2) at ($(nx1) + (0.0,0.2)$) {};
\draw (and3) -- (dummy1);
\draw(and3) -- (dummy2);
\draw (and3) -- (x3);

\node (or) at ($(and2) + (0,1.5)$) {$+$};
\draw (or) -- (and1);
\draw (or) -- (and2);
\draw (or) -- (and3);

\end{tikzpicture}

}

\begin{document}
\twocolumn[
\icmltitle{A Semantic Loss Function for Deep Learning with Symbolic Knowledge}

\icmlsetsymbol{equal}{*}

\begin{icmlauthorlist}
\author{Jingyi Xu, Zilu Zhang, Tal Friedman, Yitao Liang \& Guy Van den Broeck\\
Computer Science Department\\
University of California, Los Angeles\\ 
Los Angeles, CA, USA\\
\texttt{\small jixu@cs.ucla.edu,zhangzilu@pku.edu.cn,\{tal,yliang,guyvdb\}@cs.ucla.edu} 
}
\icmlauthor{Jingyi Xu}{to}
\icmlauthor{Zilu Zhang}{goo}
\icmlauthor{Tal Friedman}{to}
\icmlauthor{Yitao Liang}{to}
\icmlauthor{Guy Van den Broeck}{to}

\end{icmlauthorlist}

\icmlaffiliation{to}{Department of Computer Science, University of California Los Angeles, Los Angeles, CA, USA}
\icmlaffiliation{goo}{Peking University, Beijing, China}

\icmlcorrespondingauthor{Tal Friedman}{tal@cs.ucla.edu}
\icmlcorrespondingauthor{Yitao Liang}{yliang@cs.ucla.edu}
\icmlcorrespondingauthor{Guy Van den Broeck}{guyvdb@cs.ucla.edu}

\icmlsetsymbol{equal}{}

\icmlkeywords{Loss Function, Logic, Semi-Supervised Learning}

\vskip 0.3in
]
\printAffiliationsAndNotice{} 

\begin{abstract}
  This paper develops a novel methodology for using symbolic
knowledge in deep learning.  From first principles, we derive a semantic loss
function that bridges between neural output vectors and logical constraints.
This loss function captures how close the neural network is to satisfying
the constraints on its output. 
An experimental evaluation shows that it effectively guides the learner to achieve (near-)state-of-the-art results on semi-supervised multi-class classification. Moreover, it significantly increases the ability of the neural network to predict structured objects, such as rankings and paths. These discrete concepts are tremendously difficult to learn, and benefit from a tight integration of deep learning and symbolic reasoning methods.
\end{abstract}

\section{Introduction}

The widespread success of representation learning raises the question of which AI tasks are amenable to deep learning, which tasks require classical model-based symbolic reasoning, and whether we can benefit from a tighter integration of both approaches.
In recent years, significant effort has gone towards various ways of using representation learning to solve tasks that were previously tackled by symbolic methods. Such efforts include neural computers or differentiable programming \citep{weston2014memory,reed2015neural,graves2016hybrid,riedel2016programming}, relational embeddings or deep learning for graph data \citep{Yang2014JointRE,Lin2015LearningEA,bordes2013translating,neelakantan2015compositional,duvenaud2015convolutional,niepert2016learning}, neural theorem proving, and learning with constraints~\citep{hu2016acl,stewart2017label,minervini2017adversarial,wang2017premise}.

This paper considers learning in domains where we have symbolic knowledge connecting the different outputs of a neural network. This knowledge takes the form of a constraint (or sentence) in Boolean logic. It can be as simple as an exactly-one constraint for one-hot output encodings, or as complex as a structured output prediction constraint for intricate combinatorial objects such as rankings, subgraphs, or paths. Our goal is to augment neural networks with the ability to learn how to make predictions subject to these constraints, and use the symbolic knowledge to improve the learning performance.

Most neuro-symbolic approaches aim to simulate or learn symbolic reasoning in an end-to-end deep neural network, or capture symbolic knowledge in a vector-space embedding. 
This choice is partly motivated by the need for smooth \emph{differentiable} models; adding symbolic reasoning code (e.g., SAT solvers) to a deep learning pipeline destroys this property. 
Unfortunately, while making reasoning differentiable, the precise logical meaning of the knowledge is often lost.
In this paper, we take a distinctly unique approach, and tackle the problem
of differentiable but sound logical reasoning from first principles. Starting
from a set of intuitive axioms, we derive the differentiable \emph{semantic
  loss} which captures how well the outputs of a neural network match a given constraint. This function precisely captures the \emph{meaning} of the constraint, and is independent of its \emph{syntax}.

Next, we show how semantic loss gives significant \emph{practical improvements} in semi-supervised classification.  
In this setting, semantic loss for the exactly-one constraint permits
us to obtain a learning signal from vast amounts of unlabeled data. The key idea
is that semantic loss helps us improve how confidently we are able to classify the unlabeled data. This simple addition to the loss function of standard deep learning architectures yields (near-)state-of-the-art performance in semi-supervised classification on MNIST, FASHION, and CIFAR-10 datasets. 

\begin{figure*}
\centering
  \includegraphics[height=0.19\linewidth]{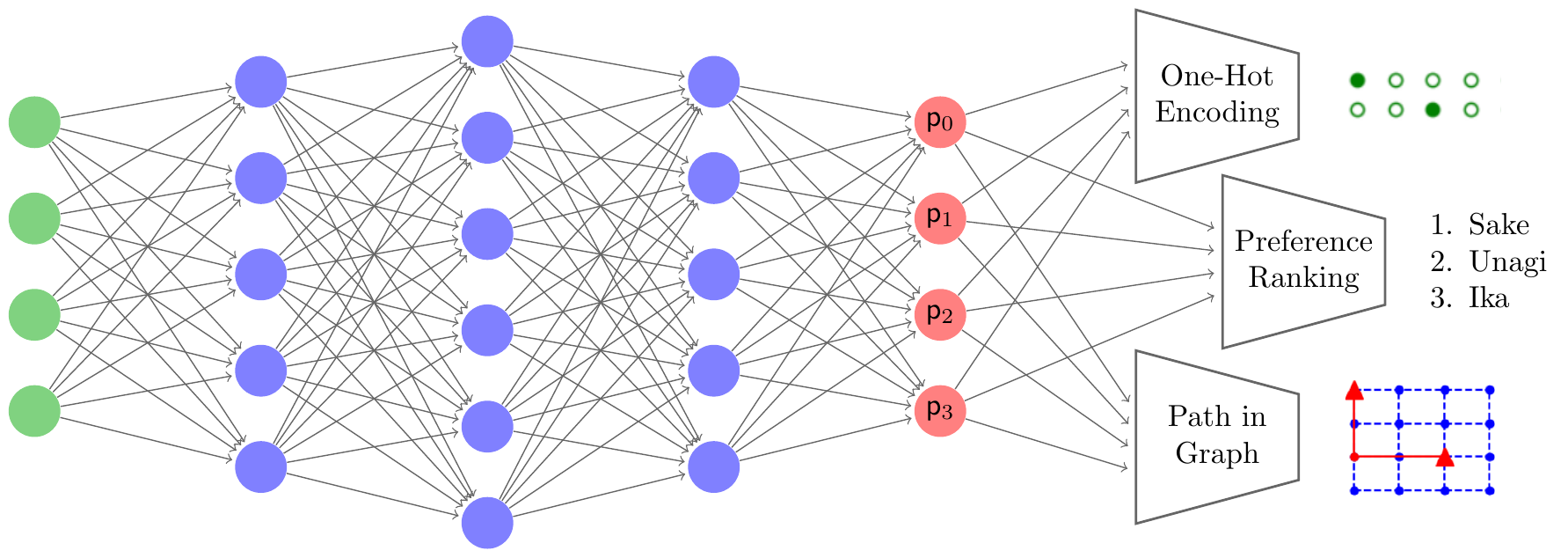}
   \caption{Outputs of a neural network feed into semantic loss functions for constraints representing
     a one-hot encoding, a total ranking of preferences, and paths in a grid graph.}
   \label{figure:neuralnetvis}
\end{figure*}

Our final set of experiments study the benefits of semantic loss for learning tasks with highly structured output, such as preference learning and path prediction in a graph~\citep{daume2009search,chang2013constrained,choi2015tractable,graves2016hybrid}. In these scenarios, the task is two-fold: learn both the structure of the output space, and the actual classification function within that space.
By capturing the structure of the output space with logical constraints, and
minimizing semantic loss for this constraint during learning, we are able to learn networks that are much more likely to correctly predict structured objects.

\section{Background and Notation}

To formally define semantic loss, we make use of concepts in propositional logic. 
We write uppercase letters ($X$,$Y$) for Boolean variables and lowercase letters ($x$,$y$) for their instantiation ($X=0$ or $X=1$). 
Sets of variables are written in bold uppercase ($\Xs$,$\Ys$), and their joint instantiation in bold lowercase~($\xs$,$\ys$). 
A literal is a variable ($X$) or its negation ($\neg X$).
A logical sentence ($\alpha$ or $\beta$) is constructed in the usual way, from variables and logical connectives ($\land$, $\lor$, etc.), and is also called a formula or constraint.
A state or world $\xs$ is an instantiation to all variables $\Xs$. A state $\xs$ satisfies a sentence $\alpha$, denoted $\xs \models \alpha$, if the sentence evaluates to be true in that world, as defined in the usual way.
A sentence $\alpha$ entails another sentence $\beta$, denoted
$\alpha \models \beta$ if all worlds that satisfy $\alpha$ also satisfy $\beta$. A sentence $\alpha$ is logically equivalent
to sentence $\beta$, denoted $\alpha \equiv \beta$, if both $\alpha \models \beta$ and $\beta \models \alpha$.

The output row vector of a neural net is denoted $\pv$. Each value in $\pv$
represents the probability of an output and falls in $[0,1]$. We use both softmax and sigmoid units for our output activation functions.
The notation for states $\xs$ is used to refer the an assignment, the logical sentence enforcing that assignment, or the binary output vector capturing that same assignment, as these are all equivalent notions.

Figure~\ref{figure:neuralnetvis} illustrates the three different concrete output constraints of varying
difficulty that are studied in our experiments. 
First, we examine the exactly-one or \emph{one-hot constraint} capturing the encoding used in multi-class classification. It states
that for a set of indicators $\Xs = \{X_1,\dots,X_n\}$, one and exactly one of those indicators must be
true, with the rest being false. 
This is enforced through a logical constraint $\alpha$ by conjoining sentences of the form $\neg X_1 \lor \neg X_2$ for all pairs of variables (at most one variable is true), and a single sentence $X_1 \lor \dots \lor X_n$ (at least one variable is true).
Our experiments further examine the \emph{valid simple path
constraint}. It states for a given source-destination pair and edge
indicators that the edge indicators set to true must form a valid simple
path from source to destination. Finally, we explore the \emph{ordering constraint},
which requires that a set of $n^2$ indicator variables represent a total
ordering over $n$ variables, effectively encoding a permutation matrix.
For a full description of the path and ordering
constraints, we refer to Section~\ref{section:complex}.

\section{Semantic Loss}
In this section, we formally introduce semantic loss. 
We begin by giving the definition and our intuition behind it.
This definition itself provides all of the necessary mechanics for enforcing
constraints, and is sufficient for the understanding of our
experiments in Sections 4 and 5. 
We also show that semantic loss is not just an
arbitrary definition, but rather is defined uniquely by a set of intuitive
assumptions.
After stating the assumptions formally, we then provide an axiomatic proof
of the uniqueness of semantic loss in satisfying these assumptions.

\subsection{Definition}

The semantic loss $\sloss(\alpha,\pv)$ is a function of a sentence $\alpha$ in
propositional logic, defined over variables $\Xs = \{X_1,\dots,X_n\}$, and a
vector of probabilities $\pv$ for the same variables $\Xs$. Element $\pv_i$
denotes the predicted probability of variable $X_i$, and corresponds to a single
output of the neural net. For example, the semantic loss between the one-hot
constraint from the previous section, and a neural net output vector $\pv$, is
intended to capture how close the prediction $\pv$ is to having exactly one
output set to true (i.e. 1), and all others set to false (i.e. 0),
regardless of which output is correct. The formal definition of this is as follows:
\begin{definition}[Semantic Loss] \label{def:sloss}
Let $\pv$ be a vector of probabilities, one for each variable in $\Xs$, and let $\alpha$ be a sentence over $\Xs$. The semantic loss between $\alpha$ and $\pv$ is
  \[\sloss(\alpha,\pv) \propto - \log \,\, \sum_{\xs \models \alpha} \,\,\, \prod_{i: \xs \models \Xr_i} \pv_i \,\,\, \prod_{i: \xs \models \neg \Xr_i}  (1-\pv_i).\]
\end{definition}

Intuitively, the semantic loss is proportional to a negative logarithm of the
probability of generating a state that satisfies the constraint, when sampling
values according to $\pv$. Hence, it is the self-information (or ``surprise'')
of obtaining an assignment that satisfies the
constraint~\citep{jones1979elementary}.

\subsection{Derivation from First Principles}

In this section, we begin with a theorem stating the uniqueness of semantic
loss, as fixed by a series of axioms. The full set of axioms and the derivation of the precise semantic loss function is
described in Appendix~\ref{a:axiomatization}\footnote{Appendices are included in the supplementary material.}.

\begin{theorem}[Uniqueness] \label{thm:unique}
  The semantic loss function in Definition~\ref{def:sloss} satisfies
  all axioms in Appendix~\ref{a:axiomatization} and is the only function that does so, up to a multiplicative constant.
\end{theorem}

In the remainder of this section, we provide a selection of the most intuitive axioms from Appendix~\ref{a:axiomatization}, as well as some key properties.

First, to retain logical meaning, we postulate that semantic loss is monotone in the order of implication.
\begin{axiom}[Monotonicity] \label{prop:monotonicity}
  If $\alpha \models \beta$, then the semantic loss $\sloss(\alpha,\pv) \geq \sloss(\beta,\pv)$ for any vector $\pv$.
\end{axiom}
Intuitively, as we add stricter requirements to the logical constraint, going from $\beta$ to $\alpha$ and making it harder to satisfy, the semantic loss cannot decrease. For example, when $\beta$ enforces the output of an neural network to encode a subtree of a graph, and we tighten that requirement in $\alpha$ to be a path, the semantic loss cannot decrease. Every path is also a tree and any solution to $\alpha$ is a solution to $\beta$.

A direct consequence following the monotonicity axiom is that logically equivalent sentences must incur an identical semantic loss for the same probability vector \pv. Hence, the semantic loss is indeed a semantic property of the logical sentence, and \emph{does not depend on its syntax}.
\begin{proposition} [Semantic Equivalence] \label{prop:equivalence}
  If $\alpha \equiv \beta$, then the semantic loss  $\sloss(\alpha,\pv) = \sloss(\beta,\pv)$ for any vector $\pv$.
\end{proposition}

Another consequence is that semantic loss must be non-negative if we want the loss to be 0 for a true sentence.

Next, we state axioms establishing a correspondence between logical constraints and data.
A state $\xs$ can be equivalently represented as both a binary data vector, as well as a logical constraint that enforces a value for every variable in $\Xs$. 
When both the constraint and the predicted vector represent the same state (for example, $X_1 \land \neg X_2 \land X_3$ vs.~$[1\,0\,1]$), there should be no semantic~loss.
\begin{axiom}[Identity] \label{ex:identity}
  For any state $\xs$, there is zero semantic loss between its representation as a sentence, and its representation as a deterministic vector: $\forall \xs, \sloss(\xs,\xs) = 0$. 
\end{axiom}

The axiom above together with the monotonicity axiom imply that any vector satisfying the constraint must incur zero loss. For example, when our constraint $\alpha$ requires that the output vector encodes an arbitrary total ranking, and the vector $\xs$ correctly represents a single specific total ranking, there is no semantic loss.
\begin{proposition}[Satisfaction] \label{prop:satisfaction}
  If $\xs \models \alpha$, then the semantic loss  $\sloss(\alpha,\xs) = 0$.
\end{proposition}

As a special case, logical literals ($X$ or $\neg X$) constrain a single variable to take on a value, and thus play a role similar to the labels used in supervised learning. Such constraints require an even tighter correspondence: the semantic loss must act like a classical loss function (i.e., cross entropy).
\begin{axiom}[Label-Literal Correspondence] \label{ax:labelcorrespondence}
  The semantic loss of a single literal is proportionate to the cross-entropy loss for the equivalent data label: $\sloss(X,p)  \propto -\log(p)$ and $\sloss(\neg X,p)  \propto -\log(1-p)$.
\end{axiom}

Appendix~\ref{a:axiomatization} states additional axioms that allow us to prove the following form of the semantic loss for a state $\xs$.
\begin{lemma}\label{l:slossstate}
For state $\xs$ and vector $\pv$, we have $\sloss(\xs,\pv) \propto -\sum_{i: \xs \models \Xr_i} \log \pv_i - \sum_{i: \xs \models \neg \Xr_i} \log (1-\pv_i).$
\end{lemma}

Lemma~\ref{l:slossstate} falls short as a full definition of semantic loss for arbitrary sentences. One can define additional axioms to pin down $\sloss$. For example, the following axiom is satisfied by Definition~\ref{def:sloss}, and is highly desirable for learning.
\begin{axiom}[Differentiability]
  For any fixed $\alpha$, the semantic loss $\sloss(\alpha,\pv)$ is monotone in each probability in $\pv$, continuous and differentiable.
\end{axiom}

Appendix~\ref{a:axiomatization}  makes the notion of semantic loss precise by stating one additional axiom. It is based on the observation that the state loss of Lemma~\ref{l:slossstate} is proportionate to a log-probability. In particular, it corresponds to the probability of obtaining state $\xs$ after independently sampling each $X_i$ with probability $\pv_i$.
We have now derived the semantic loss function from first principles, and
arrived at Definition~\ref{def:sloss}. Moreover, we can show that
Theorem~\ref{thm:unique} holds - that it is the only choice of such a loss function.

\section{Semi-Supervised Classification} \label{s:semi}

The most straightforward constraint that is ubiquitous in classification is mutual exclusion over one-hot-encoded outputs. That is, for a given example, exactly one
class and therefore exactly one binary indicator must be true.
The machine learning community has made great strides on this task, due to the
invention of assorted deep learning representations and their associated
regularization terms \citep{alex2012nips, he2016cvpr}. 
Many of these models take large amounts of
labeled data for granted, and big data is indispensable for discovering accurate representations \citep{hastie2009overview}. 
To sustain this progress, and alleviate the need for more labeled data, there is a growing interest into utilizing unlabeled data to augment the predictive power of classifiers \citep{stewart2017label,
  bilenko2004integrating}. 
This section shows why semantic loss naturally qualifies for this task.

\begin{figure}[t]
\centering
  \begin{subfigure}[b]{0.23\textwidth} \centering
    \includegraphics[width=1.1\textwidth]{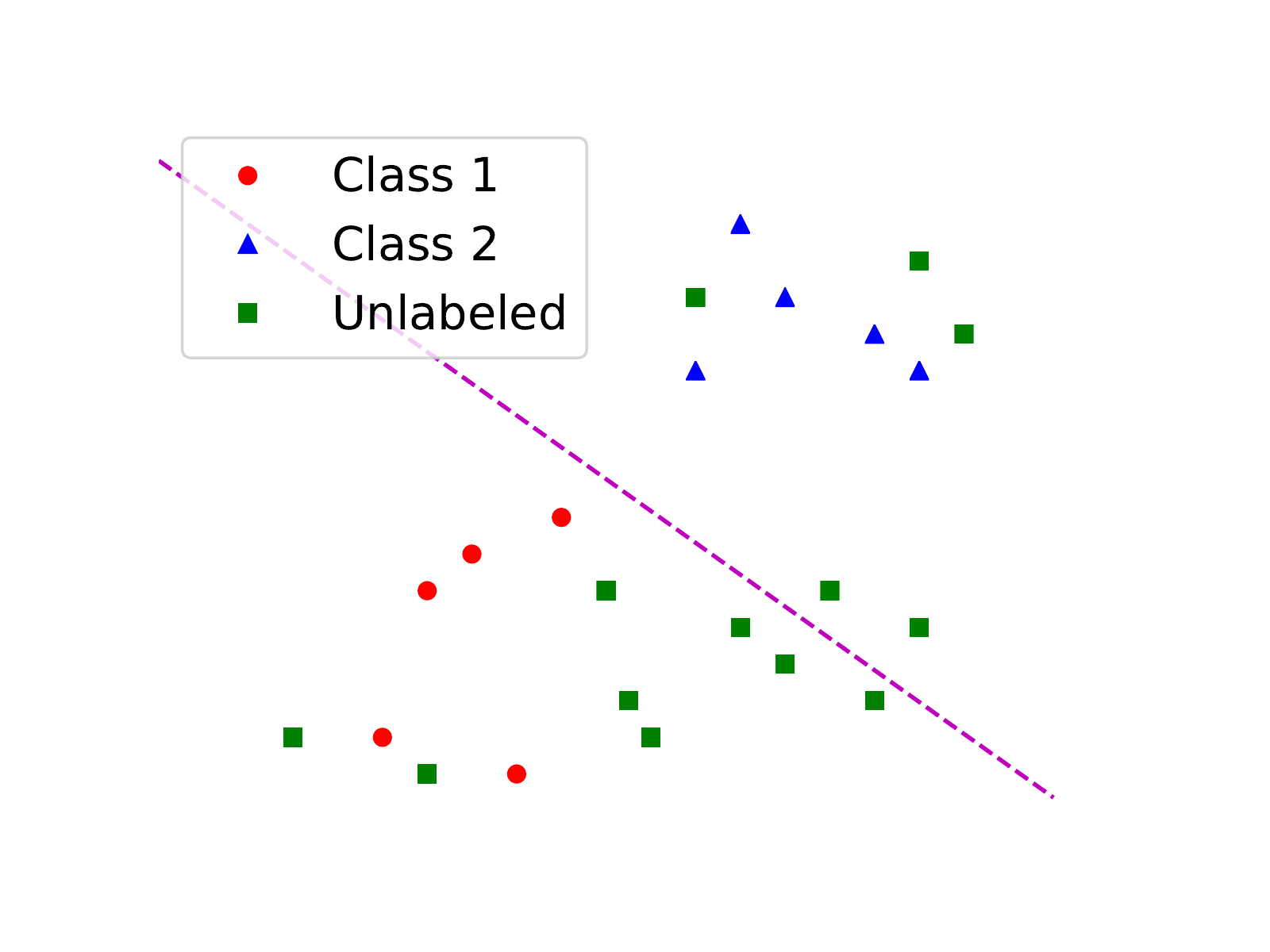}
    \caption{Trained w/o semantic loss} \label{figure: toy example:without}
  \end{subfigure}
  \begin{subfigure}[b]{0.23\textwidth} \centering
    \includegraphics[width=1.1\textwidth]{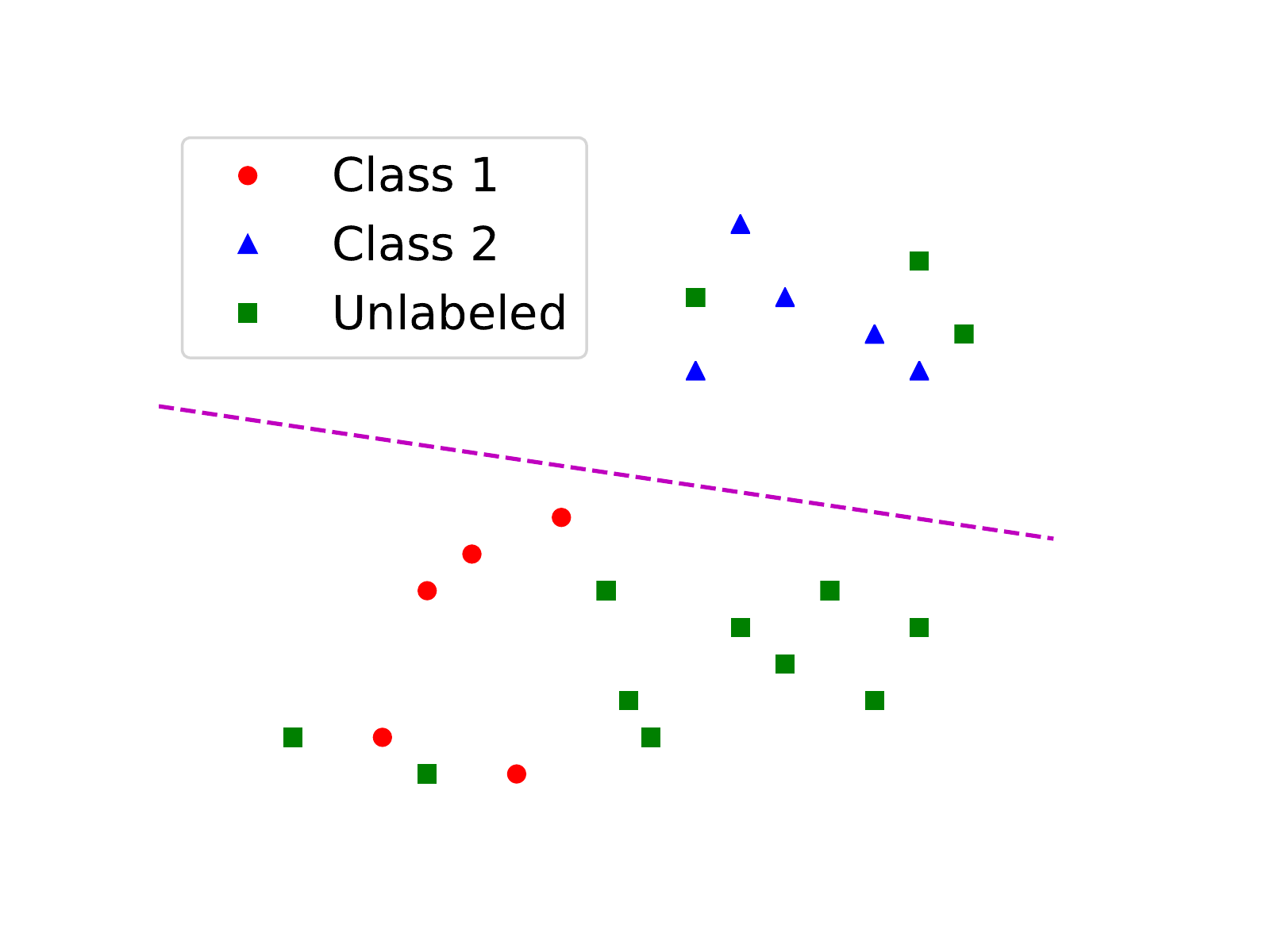}
    \caption{Trained with semantic loss} \label{figure: toy example:with}
  \end{subfigure}
   \caption{Binary classification toy example: a linear classifier without and with semantic loss.}
   \label{figure: toy example}
\end{figure}

\paragraph{Illustrative Example}
To illustrate the benefit of semantic loss in the semi-supervised setting, we
begin our discussion with a small toy example. Consider a binary classification
task; see Figure~\ref{figure: toy example}. Ignoring the unlabeled
examples, a simple linear classifier learns to distinguish the two classes by
separating the labeled examples (Figure~\ref{figure: toy example:without}).
However, the unlabeled examples are also informative, as they must carry some
properties that give them a particular label. This is the crux of semantic loss for semi-supervised learning:
a model must confidently assign a consistent class even to unlabeled data.
Encouraging the model to do so results in a more accurate decision boundary (Figure~\ref{figure: toy example:with}).

\subsection{Method} \label{s:algo}

Our proposed method intends to be generally applicable and compatible with any feedforward neural net. Semantic loss is simply another regularization term that can directly be plugged into an existing loss function. More specifically, with some weight $w$, the new overall loss becomes
\begin{align*}
\text{existing loss} + w\cdot\text{semantic loss}.
\end{align*}
When the constraint over the output space is simple (for example, there is a small number of solutions $\xs \models \alpha$), semantic loss can be directly computed 
using Definition~\ref{def:sloss}. Concretely, for the exactly-one constraint used in $n$-class classification, semantic loss reduces to
\begin{equation*}
\sloss({\text{exactly-one}},\pv) \propto - \log \, \sum_{i=1}^{n}{\pv_i}\prod_{j=1, j\neq i}^{n}{(1-\pv_j)},
\end{equation*}
where values $\pv_i$ denote the probability of class $i$ as predicted by the
neural net. Semantic loss for the exactly-one constraint is efficient and causes no noticeable computational overhead in our experiments.

In general, for arbitrary constraints $\alpha$, semantic loss is not efficient to compute using Definition~\ref{def:sloss}, and more advanced automated reasoning is required. Section~\ref{section:complex} discusses this issue in more detail. For example, using automated reasoning can reduce the time complexity to compute semantic loss for the exactly-one constraint from $O(n^2)$ (as shown above), to $O(n)$.

\subsection{Experimental Evaluation}

In this section, we evaluate semantic loss in the semi-supervised setting by
comparing it with several competitive models.\footnote{The code to reproduce all the experiments in this paper can be found at 
\url{https://github.com/UCLA-StarAI/Semantic-Loss/.}}  
As most semi-supervised learners build on a supervised learner, changing the underlying model significantly affects the semi-supervised learner's performance. For comparison, we add semantic loss to the same base models used in ladder nets \citep{rasmus2015semi}, which currently achieves state-of-the-art results on semi-supervised MNIST and CIFAR-10 \citep{krizhevsky2009cifar}.
Specifically, the MNIST base model is a fully-connected multilayer perceptron (MLP), with layers of size 784-1000-500-250-250-250-10. On CIFAR-10, it is a 10-layer convolutional neural network (CNN) with 3-by-3 padded filters. After every 3 layers, features are subject to a 2-by-2 max-pool layer with strides of 2. Furthermore, we use ReLu \citep{nair2010rectified}, batch normalization \citep{ioffe2015batch}, and Adam optimization \citep{kingma2015adam} with a learning rate of 0.002. We refer to Appendix~\ref{appendix:specification} and~\ref{appendix:tuning} for a specification of the CNN model and additional details about hyper-parameter tuning.

For all semi-supervised experiments, we use the standard 10,000 held-out test examples provided in the original datasets and randomly pick 10,000 from the standard 60,000 training examples (50,000 for CIFAR-10) as validation set. For values of $N$ that depend on the experiment, we retain $N$ randomly chosen labeled examples from the training set, and remove labels from the rest. We balance classes in the labeled samples to ensure no particular class is over-represented. 
Images are preprocessed for standardization and Gaussian noise 
is added to every pixel ($\sigma=0.3$). 

\paragraph{MNIST}
The permutation invariant MNIST classification task is commonly used as a test-bed for general semi-supervised learning algorithms. This setting does not use any prior information about the spatial arrangement of the input pixels. Therefore, it excludes many data augmentation techniques that involve geometric distortion of images, as well as convolutional neural networks.

\begin{table*}[t]
      \ra{1.05}
\vskip 0.05in
\caption {MNIST. Previously reported test accuracies followed by baselines and semantic loss results ($\pm$ stddev) }
\label{table:mnist}
\centering
{\small
\begin{tabular}{@{} l l l l @{} }
Accuracy \% with \# of used labels & 100 & 1000 & ALL \\
\midrule \midrule
AtlasRBF \citep{pitelis2014semi}   & 91.9~~($\pm$0.95) & 96.32 ($\pm$0.12) & 98.69\\
Deep Generative \citep{dgn2014nips} &96.67($\pm$0.14) &97.60 ($\pm$0.02)  & 99.04 \\
Virtual Adversarial \citep{miyato2016iclr} & 97.67 & 98.64 & 99.36 \\
Ladder Net \citep{rasmus2015semi} &  {\bf 98.94} ($\pm$0.37 ) & {\bf 99.16} ($\pm$0.08) & 99.43 ($\pm$0.02)\\
\midrule
Baseline: MLP, Gaussian Noise & 78.46 ($\pm$1.94) & 94.26 ($\pm$0.31) & 99.34 ($\pm 0.08$) \\
Baseline: Self-Training  & 72.55 ($\pm$4.21) & 87.43 ($\pm$3.07) & \\
Baseline: MLP with Entropy Regularizer & 96.27 ($\pm$0.64) & 98.32 ($\pm$0.34) & 99.37 ($\pm$0.12)  \\
\midrule
 MLP with Semantic Loss  & 98.38 ($\pm$0.51) & 98.78 ($\pm$0.17)  & 99.36 ($\pm$0.02)
\end{tabular}
}
\vskip -0.05in
\end{table*}

\begin{table*}[t]
      \ra{1.05}
\centering
\caption{FASHION. Test accuracy comparison between MLP with semantic loss and ladder nets.}
\label{table:fashion}
\vskip -0.05in
{\small
\begin{tabular}{ @{}l l l l l @{}}
Accuracy \% with \# of used labels & 100 & 500 &1000 & ALL \\
\midrule\midrule
Ladder Net  \citep{rasmus2015semi}& 81.46 ($\pm$0.64 ) &  85.18 ($\pm$0.27) &  86.48 ($\pm$0.15) &  90.46 
\\
\midrule
Baseline: MLP, Gaussian Noise & 69.45 ($\pm 2.03$) & 78.12 ($\pm 1.41$) & 80.94 ($\pm 0.84$) &89.87 \\
 MLP with Semantic Loss &  {\bf 86.74} ($\pm 0.71$) & {\bf 89.49} ($\pm 0.24$)  & {\bf 89.67} ($\pm 0.09$) & 89.81
\end{tabular}
}
\vskip -0.1in
\end{table*}

When evaluating on MNIST,  we run experiments for 20 epochs, with a batch size of 10. Experiments are repeated 10 times with different random seeds. 
Table~\ref{table:mnist} compares semantic loss to three baselines and state-of-the-art results from the literature. 
The first baseline is a purely supervised MLP, which makes no use of unlabeled
data. The second is the classic self-training method for semi-supervised
learning, which operates as follows. After every 1000 iterations, the unlabeled
examples that are predicted by the MLP to have more than $95\%$ probability of
belonging to a single class, are assigned a pseudo-label and become labeled
data. 

Additionally, we constructed a third baseline by replacing the semantic loss term with the entropy regularizor described in \citet{grandvalet2005semi} as a direct comparison for semantic loss. With the same amount of parameter tuning, we found that using entropy achieves an accuracy of $96.27\%$ with 100 labeled examples, and $98.32\%$ with 1000 labelled examples, both are slightly worse than the accuracies reached by semantic loss. Furthermore, to our best knowledge, there is no straightforward method to generalize entropy loss to the settings of complex constraints, where semantic loss is clearly defined and can be easily deployed. We will discuss this more in Section \ref{section: complex constraints}.

Lastly, We attempted to create a fourth baseline by constructing a constraint-sensitive loss term in the style of \citet{hu2016acl}, using a simple extension of
Probabilistic Soft Logic (PSL) \citep{Kimmig2012ASI}. PSL translates logic into continuous domains by using soft truth values, and defines functions in the
real domain corresponding to each Boolean function. This is normally done for Horn clauses, but since they are not sufficiently expressive for our constraints, we apply fuzzy operators to arbitrary sentences instead.
We are forced to deal with a key difference between semantic loss and PSL: encodings in fuzzy logic
are highly sensitive to the syntax used for the constraint (and therefore violate Proposition~\ref{prop:equivalence}). We selected two reasonable encodings detailed in Appendix~\ref{a:psl}. The
first encoding results in a constant value of 1, and thus could not be used for
semi-supervised learning. The second encoding empirically deviates from 1 by
$<0.01$, and since we add Gaussian noise to the pixels, no amount of tuning was able to extract meaningful supervision. Thus, we do not report these
results.

When given 100 labeled examples ($N=100$), MLP with semantic loss gains around $20\%$ improvement over the purely supervised baseline. The improvement is even larger ($25\%$) compared to self-training. Considering \emph{the only change is an additional loss term}, this result is very encouraging.
Comparing to the state of the art, ladder nets slightly outperform semantic loss by $0.5\%$ accuracy. 
This difference may be an artifact of the excessive tuning of architectures, hyper-parameters and learning rates that the MNIST dataset has been subject to. 
In the coming experiments, we extend our work to more challenging
datasets, in order to provide a clearer comparison with ladder~nets.  
Before that, we want to share a few more thoughts on how semantic loss works.
A classical softmax layer interprets its output as representing a categorical distribution. Hence, by normalizing its outputs, softmax enforces the same mutual exclusion constraint enforced in our semantic loss function. 
However, there does not exist a natural way to extend softmax loss to unlabeled samples. In contrast, semantic loss does provide a learning signal on unlabeled samples, by forcing the underlying classifier to make an decision and construct a confident hypothesis for all data.  
However, for the fully supervised case ($N=\text{all}$), semantic loss does not
significantly affect accuracy. Because the MLP has enough capacity to almost
perfectly fit the training data, where the constraint is always satisfied,
semantic loss is almost always zero. This is a direct consequence of
Proposition~\ref{prop:satisfaction}.

\paragraph{FASHION}

The FASHION \citep{xiao2017online} dataset consists of Zalando's article images, aiming to serve as a more challenging drop-in replacement for MNIST.  Arguably, it has not been overused and requires more advanced techniques to achieve good performance.  
As in the previous experiment, we run our method for 20 epochs, whereas ladder
nets need 100 epochs to converge. Again, experiments are repeated 10 times and
Table~\ref{table:fashion} reports the classification accuracy and its standard deviation (except for $N=\text{all}$ where it is close to $0$ and omitted for space).

Experiments show that utilizing semantic loss results in a very large
$17\%$ improvement over the baseline when only 100 labels are provided. Moreover, our method compares favorably to ladder nets, except when the setting degrades to be fully supervised. 
Note that our method already nearly reaches its maximum accuracy with 500 labeled examples, which is only $1\%$ of the training dataset.

\begin{table}[t]
\ra{1.05}
\caption {CIFAR. Test accuracy comparison between CNN with Semantic Loss and ladder nets.}
\label{table:cifar}
\vskip -0.05in
\centering
{\small
\begin{tabular}{@{} l l l @{}}
Accuracy \% with \# of used labels & 4000 & ALL\\
\midrule \midrule
CNN Baseline in Ladder Net & 76.67 ($\pm$0.61)& 90.73 \\
Ladder Net \citep{rasmus2015semi} & 79.60 ($\pm$0.47) &\\
\midrule
Baseline: CNN, Whitening, Cropping & 77.13 & 90.96\\
CNN with Semantic Loss  & {\bf 81.79} & 90.92
\end{tabular}
}
\vskip -0.1in
\end{table}

\paragraph{CIFAR-10} To show the general applicability of semantic loss, we evaluate it on CIFAR-10. This dataset consisting of 32-by-32 RGB images in 10 classes. A simple MLP would not have enough representation power to capture the huge variance across objects within the same class. To cope with this spike in difficulty, we switch our underlying model to a 10-layer CNN as described earlier.
We use a batch size of 100 samples of which half are unlabeled. Experiments are run for 100 epochs. However, due to our limited computational resources, we report on a single trial. 
Note that we make slight modifications to the underlying model used in ladder nets to reproduce similar baseline performance. Please refer to Appendix~\ref{appendix:specification} for further details.

As shown in Table~\ref{table:cifar}, our method compares favorably to ladder nets. However, due to the slight difference in performance  between the supervised base models, a direct comparison would be methodologically flawed. Instead, we compare the net improvements over baselines. In terms of this measure, our method scores a gain of $4.66\%$ whereas ladder nets gain $2.93\%$.

\subsection{Discussion}
The experiments so far have demonstrated 
the competitiveness and general applicability of our proposed method on
semi-supervised learning tasks. It surpassed the previous state of the art
(ladder nets) on FASHION and CIFAR-10, while being close on MNIST. Considering
the simplicity of our method, such results are encouraging. Indeed, a key
advantage of semantic loss is that it only requires a simple additional loss
term, and thus incurs almost no computational overhead. 
Conversely, this property makes our method sensitive to the underlying model's performance.

Without the underlying predictive power of a strong supervised learning model, we do not expect to see the same benefits we observe here. 
Recently, we became aware that \citet{miyato2016iclr} extended their work to CIFAR-10 and achieved state-of-the-art results \citep{miyato2017}, surpassing our performance by $5\%$. In future work, we plan to investigate whether applying semantic loss on their architecture would yield an even stronger performance.

Figure~\ref{figure: fashion} in the appendix illustrates the effect of semantic loss on FASHION pictures whose correct label was hidden from the learner.
Pictures~\ref{figure: fashion:correct_confident} and~\ref{figure: fashion:correct_unconfident} are correctly classified by the supervised base model, and on the first set it is confident about this prediction ($\pv_i > 0.8$). 
Semantic loss rarely diverts the model from these initially correct labels. However, it bootstraps these unlabeled examples to achieve higher confidence in the learned concepts. 
With this additional learning signal, the model changes its beliefs about Pictures~\ref{figure: fashion:incorrect_unconfident}, which it was previously uncertain about. 
Finally, even on confidently misclassified Pictures~\ref{figure: fashion:incorrect_confident}, semantic loss is able to remedy the mistakes of the base model.

\section{Learning with Complex Constraints}
\label{section: complex constraints}

\label{section:complex}
While much of current machine learning research is focused on problems such as multi-class classification, there remain a multitude of difficult problems involving highly constrained output domains. As mentioned in the previous section, semantic loss has little effect on the fully-supervised exactly-one classification problem. 
This leads us to seek out more difficult problems to illustrate that semantic
loss can also be highly informative in the supervised case, provided the output
domain is a sufficiently complex space. Because semantic loss is defined by a
Boolean formula, it can be used on any output domain that can be fully described
in this manner. Here, we develop a framework for making semantic loss tractable on highly complex constraints, and evaluate it on some difficult examples.

\begin{figure}[t]
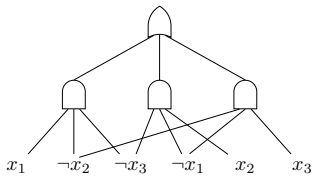

\centering
\scalebox{0.76} {
  \toycircuit{}
}
\caption{A compiled decomposable and deterministic circuit for the exactly-one constraint with 3 variables.}
\label{figure:circuit}
\end{figure}

\begin{figure} [t]
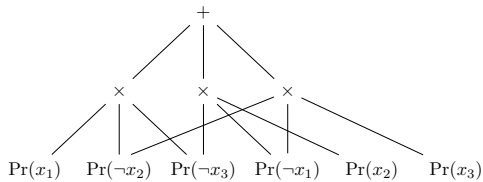

\centering
  \scalebox{0.70} {
\circuitComputation{} }
\caption{The corresponding arithmetic circuit for the exactly-one constraint
  with 3 variables.}
  \label{figure:arithmeticcircuit}
\end{figure}

\subsection{Tractability of Semantic Loss}

Our goal here is to develop a general method for computing both semantic loss and
its gradient in a tractable manner. Examining Definition~\ref{def:sloss} of
semantic loss, we see that the right-hand side is a well-known automated
reasoning task called weighted model counting
(WMC)~\citep{chavira2008,sang2005}. 

Furthermore, we know of circuit languages that compute
WMCs, and that are amenable to
backpropagation~\citep{DarwicheJACM}. We use the circuit
compilation techniques in \citet{darwiche2011sdd} to
build a Boolean circuit representing semantic loss. 
We refer to the literature for details of this compilation approach.
Due to certain properties of this circuit
form, we can use it to compute both the values and the gradients of semantic loss in time
linear in the size of the circuit \citep{darwicheJAIR02}. 
Once constructed, we can add it to our standard loss function as described in
Section~\ref{s:algo}.

Figure~\ref{figure:circuit} shows an example Boolean circuit for the exactly-one
constraint with 3 variables. We begin with the standard logical encoding for the
exactly-one constraint $(x_1 \lor x_2 \lor x_3) \land (\neg x_1 \lor \neg x_2)
\land (\neg x_1 \land \neg x_3) \land (\neg x_2 \land \neg x_3)$, and then
compile it into a circuit that can perform WMC efficiently \citep{chavira2008}. 
The cost of this step depends on the type of the constraint: for bounded-treewidth constraints it can be done efficiently, and for some constraints exact compilation is theoretically hard. In that case, we have to rely on advanced knowledge compilation algorithms to still perform this step efficiently in practice.
Our semantic loss framework can be applied regardless of how the circuit gets compiled. 
On our example, following the circuit bottom up, the logical
function can be read as $(x_1 \land \neg x_2 \land \neg x_3) \lor (\neg x_1
\land x_2 \land \neg x_3) \lor (\neg x_1 \land \neg x_2 \land x_3)$. Once this
Boolean circuit is built, we can convert it to an arithmetic circuit, by simply
changing {\tt AND} gates into $*$, and {\tt OR} gates into $+$, as shown in
Figure~\ref{figure:arithmeticcircuit}. Now, by pushing the probabilities up
through the arithmetic circuit, evaluating the root gives the probability of the logical
formula described by the Boolean circuit -- this is precisely the exponentiated semantic loss.
Notice that this computation was not possible with the Boolean formula we began with: it is a direct result of our circuit having two key properties called determinism and decomposability. Finally, we can similarly do another pass down on the circuit to compute partial
derivatives \citep{darwicheJAIR02}.

\subsection{Experimental Evaluation}
Our ambition when evaluating semantic loss'~performance on complex constraints
is not to achieve state-of-the-art performance on any particular problem, but
rather to highlight its effect. To
this end, we evaluate our method on problems with a difficult output space, where the model could no longer be fit directly from data, and
purposefully use simple MLPs for evaluation. We want to emphasize that the constraints used in this evaluation are intentionally designed to be very difficult; much more so than the simple implications that are usually studied (e.g., \citet{hu2016acl}). 
Hyper-parameter tuning details are again in Appendix~\ref{appendix:tuning}.

\paragraph{Grids}
We begin with a classic algorithmic problem: finding the shortest path in a graph. Specifically, we use a 4-by-4 grid $G=(V,E)$  with uniform edge weights. We randomly remove edges for each example to increase difficulty. Formally, our input is a
binary vector of length $|V| + |E|$, with the first $|V|$ variables indicating sources and destinations, and the next $|E|$ which edges are removed. Similarly, each
label is a binary vector of length $|E|$ indicating which
edges are in the shortest path. Finally, we require through our constraint $\alpha$ that
the output form a valid simple path between the desired source and destination.
To compile this constraint, we use the method of
\citet{NishinoYMN17a} to encode pairwise simple paths, and enforce the correct source and destination. For more details on the
constraint and data generation process, see
Appendix~\ref{a:complex}.

To evaluate, we use a dataset of 1600 examples, with a 60/20/20
train/validation/test split. 
Table~\ref{tab:gridres} compares test
accuracy between a 5-layer MLP baseline, and the same model augmented with
semantic loss. 
We report three different accuracies that illustrate the effect of semantic
loss: ``Coherent'' indicates the percentage of examples for which the classifier
gets the entire configuration right, while ``Incoherent'' measures the
percentage of individually correct binary labels, which as a whole may not
constitute a valid path at all. Finally, ``Constraint'' describes the percentage
of predictions given by the model that satisfy the constraint associated with
the problem.
In the case of incoherent accuracy, semantic loss has little effect, and in fact
slightly reduces the accuracy as it combats the standard sigmoid cross entropy.
In regard to coherent accuracy however, semantic loss has a very large
effect in guiding the network to jointly learn true paths, rather than
optimizing each binary output individually. We further see this by observing the
large increase in the percentage of predictions that really are paths between
the desired nodes in the graph.

\begingroup
\begin{table}[t]
\ra{1.05}
\centering
\caption {Grid shortest path test results: coherent, incoherent and constraint accuracy.}
\vskip -0.05in
\label{tab:gridres}
{\small
\begin{tabular}{ @{} l l l l }
Test accuracy \%  & Coherent & Incoherent & Constraint \\
\midrule \midrule
5-layer MLP & 5.62 & {\bf 85.91} & 6.99 \\
\midrule
+ Semantic loss & {\bf 28.51} & 83.14 & {\bf 69.89} \\
\end{tabular}
\vskip -0.05in
}
\end{table}

\begin{table}[t]
\ra{1.05}
\centering
\caption {Preference prediction test results: coherent, incoherent and constraint accuracy.}
\vskip -0.05in
\label{tab:prefres}
{\small
\begin{tabular}{@{}l l l l @{}}
Test accuracy \%  & Coherent & Incoherent & Constraint\\
\midrule \midrule
3-layer MLP & 1.01 & {\bf 75.78} & 2.72 \\
\midrule
+ Semantic loss & {\bf 13.59} & 72.43 & {\bf 55.28} \\
\end{tabular}
}
\vskip -0.1in
\end{table}
\endgroup

\paragraph{Preference Learning}
\label{subsubsection: preference learning}
The next problem is that of predicting a complete order of
preferences. That is, for a given set of user features, we want to predict
how the user ranks their preference over a fixed set of items. We encode
a preference ordering over \(n\) items as a flattened binary matrix $\left\{X_{ij}\right\}$, where for each \(i, j \in \{1,\ldots,n\}\), $X_{ij}$ denotes that
item \(i\) is at position \(j\) \citep{choi2015tractable}. Clearly, not all configurations 
of outputs correspond to a valid ordering, so our constraint allows only for those that are.

We use preference ranking data over 10 types of sushi for 5000
individuals, taken from \mbox{\textsc{PrefLib}} \citep{MaWa13a}. We take the ordering over
6 types of sushi as input features to predict the ordering over the remaining 4
types, with splits identical to those in \citet{ShenChoiDarwiche17}.
We again split the data 60/20/20 into train/test/split, and employ a 3-layer MLP
as our baseline. Table~\ref{tab:prefres} compares the baseline to the same MLP
augmented with semantic loss for valid total orderings. Again, we see that semantic
loss has a marginal effect on incoherent accuracy, but significantly improves the
network's ability to predict valid, correct orderings. Remarkably, without semantic loss, the network is only able to output a valid ordering on $1\%$ of examples.

\section{Related Work}

Incorporating symbolic background knowledge into machine learning is a long-standing challenge~\citep{srinivasan1995comparing}.
It has received considerable attention for structured prediction in natural language processing, in both supervised and semi-supervised settings. For example, \emph{constrained conditional models} extend linear models with constraints that are enforced through integer linear programming~\citep{chang2008learning,chang2013constrained}.
Constraints have also been studied in the context of probabilistic graphical models~\citep{mateescu2008mixed,ganchev2010posterior}. \citet{KisaVCD14} utilize a circuit language called the \emph{probabilistic sentential decision diagram} to induce distributions over arbitrary logical formulas. They learn generative models that satisfy preference and path constraints~\citep{choi2015tractable,ChoiTavabiDarwiche16}, which we study in a discriminative setting.

Various deep learning techniques have been proposed to enforce either arithmetic constraints~\citep{pathak2015constrained,marquez2017imposing} or logical constraints~\citep{rocktaschel2015injecting,hu2016acl,demeester2016lifted,stewart2017label,minervini2017adversarial,diligenti2017semantic,donadello2017logic} on the output of a neural network. The common approach is to reduce logical constraints into differentiable arithmetic objectives by replacing logical operators with their fuzzy t-norms and logical implications with simple inequalities. A downside of this fuzzy relaxation is that the logical sentences lose their precise meaning. 
The learning objective becomes a function of the syntax rather than the semantics (see Section~\ref{s:semi}). Moreover, these relaxations are often only applied to Horn clauses. One alternative is to encode the logic into a factor graph and perform loopy belief propagation to compute a loss function~\citep{naradowsky2017modeling}, which is known to have issues in the presence of complex logical constraints~\citep{smith2014loopy}.

Several specialized techniques have been proposed to exploit the rich structure of real-world labels. 
\citet{deng2014eccv} propose hierarchy and exclusion graphs that jointly model hierarchical categories. It is a method invented to address examples whose labels are not provided at the most specific level. 
Finally, the objective of semantic loss to increase the confidence of
predictions on unlabeled data is related to information-theoretic approaches
to semi-supervised learning~\citep{grandvalet2005semi,erkan2010pmlr}, and
approaches that increase robustness to output
perturbation~\citep{miyato2016iclr}.
A key difference between semantic loss and these information-theoretic losses is that semantic loss generalizes to arbitrary logical output constraints that are  much more complex.

\section{Conclusions \& Future Work}
Both reasoning and semi-supervised learning are often identified as key challenges for deep learning going forward.
In this paper, we developed a principled way of combining automated reasoning for propositional logic with existing deep learning architectures. Moreover, we showed that semantic loss provides significant benefits during semi-supervised classification, as well as deep structured prediction for highly complex output spaces.

An interesting direction for future work is to come up with effective approximations of semantic loss, for settings where even the methods we have described are not sufficient. There are several potential ways to proceed with this, including hierarchical abstractions, relaxations of the constraints, or projections on random subsets of variables.

\subsection*{Acknowledgements}
 This research was conducted while Zilu Zhang was a visiting student at StarAI Lab, UCLA. The authors thank Arthur Choi and Yujia Shen for helpful discussions.  This work is partially supported by NSF grants \#IIS-1657613, \#IIS-1633857 and DARPA XAI grant \#N66001-17-2-4032. 

\bibliography{reference}
\bibliographystyle{icml2018}

\appendix

\section{Axiomatization of Semantic Loss: Details} \label{a:axiomatization}

This appendix provides further details on our axiomatization of semantic loss. We detail here a complete axiomatization of semantic loss, which will involve restating some axioms and propositions from the main paper.

The first axiom says that there is no loss when the logical constraint $\alpha$ is always true (it is a logical tautology), independent of the predicted probabilities $\pv$.
\begin{axiom}[Truth] \label{ex:truth} \label{firstaxiom}
  The semantic loss of a true sentence is zero: $\forall \pv, \sloss(\true,\pv) = 0$. 
\end{axiom}

Next, when enforcing two constraints on disjoint sets of variables, we want the
ability to compute semantic loss for the two constraints separately, and sum
the results for their joint semantic loss.
\begin{axiom}[Additive Independence] \label{ex:addindep}
  Let $\alpha$ be a sentence over $\Xs$ with probabilities $\pv$. Let $\beta$ be a sentence over $\Ys$ disjoint from $\Xs$ with probabilities $\qv$. The semantic loss between sentence $\alpha \land \beta$ and the joint probability vector $\left[\pv \, \qv\right]$ decomposes additively: $\sloss(\alpha \land \beta,\left[\pv \, \qv\right]) = \sloss(\alpha,\pv) + \sloss(\beta,\qv)$.
\end{axiom}
It directly follows from Axioms~\ref{ex:truth} and \ref{ex:addindep} that the
probabilities of variables that are not used on the constraint do not affect the
semantic loss. 

Proposition~\ref{prop:locality} formalizes this intuition.

\begin{proposition}[Locality] \label{prop:locality}
  Let $\alpha$ be a sentence over $\Xs$ with probabilities $\pv$. For any $\Ys$ disjoint from $\Xs$ with probabilities $\qv$, the semantic loss $\sloss(\alpha,\left[\pv \, \qv\right]) = \sloss(\alpha,\pv)$.
  \begin{proof}
  Follows from the additive independence and truth axioms. Set $\beta = \true$ in the additive independence axiom, and observe that this sets $\sloss(\beta,\qv) = 0$ because of the truth axiom.
  \end{proof}
\end{proposition}

To maintain logical meaning, we postulate that semantic loss is monotone in the order of implication.
\begin{axiom}[Monotonicity] \label{prop:monotonicity2}
  If $\alpha \models \beta$, then the semantic loss $\sloss(\alpha,\pv) \geq \sloss(\beta,\pv)$ for any vector $\pv$.
\end{axiom}
Intuitively, as we add stricter requirements to the logical constraint, going
from $\beta$ to $\alpha$ and making it harder to satisfy, semantic loss cannot decrease. For example, when $\beta$ enforces the output of an neural network to encode a subtree of a graph, and we tighten that requirement in $\alpha$ to be a path, semantic loss cannot decrease. Every path is also a tree and any solution to $\alpha$ is a solution to $\beta$.

A first consequence following the monotonicity axiom is that logically equivalent sentences must incur an identical semantic loss for the same probability vector \pv. Hence, the semantic loss is indeed a semantic property of the logical sentence, and \emph{does not depend on the syntax} of the sentence.
\begin{proposition}
  If $\alpha \equiv \beta$, then the semantic loss  $\sloss(\alpha,\pv) = \sloss(\beta,\pv)$ for any vector $\pv$.
\end{proposition}

A second consequence is that semantic loss must be non-negative.

\begin{proposition}[Non-Negativity] \label{prop:nonneg}
 Semantic loss is non-negative.
\begin{proof}
  Because $\alpha \models \true$ for all $\alpha$, the monotonicity axiom implies that $\forall \pv, \sloss(\alpha,\pv) \geq \sloss(\true,\pv)$. By the truth axiom, $\sloss(\true,\pv) = 0$, and therefore $\sloss(\alpha,\pv) \geq 0$ for all choices of $\alpha$ and $\pv$. 
 \end{proof}
\end{proposition}


A state $\xs$ is equivalently represented as a data vector, as well as a logical constraint that enforces a value for every variable in $\Xs$. 
When both the constraint and the predicted vector represent the same state (for example, $X_1 \land \neg X_2 \land X_3$ vs.~$[1\,0\,1]$), there should be no semantic~loss.
\begin{axiom}[Identity] \label{ex:identity2}
  For any state $\xs$, there is zero semantic loss between its representation as a sentence, and its representation as a deterministic vector: $\forall \xs, \sloss(\xs,\xs) = 0$. 
\end{axiom}

The axioms above together imply that any vector satisfying the constraint must incur zero loss. For example, when our constraint $\alpha$ requires that the output vector encodes an arbitrary total ranking, and the vector $\xs$ correctly represents a single specific total ranking, there is no semantic loss.
\begin{proposition}[Satisfaction] \label{prop:satisfaction2}
  If $\xs \models \alpha$, then the semantic loss  $\sloss(\alpha,\xs) = 0$.
\end{proposition}

\begin{proof}[Proof of Proposition~\ref{prop:satisfaction2}]
 The monotonicity axiom specializes to say that if $\xs \models \alpha$, we have that $\forall \pv, \sloss(\xs,\pv) \geq \sloss(\alpha,\pv)$. By choosing $\pv$ to be $\xs$, this implies $\sloss(\xs,\xs) \geq \sloss(\alpha,\xs)$. From the identity axiom, $\sloss(\xs,\xs) = 0$, and therefore $0 \geq \sloss(\alpha,\xs)$. Proposition~\ref{prop:nonneg} bounds the loss from below as $\sloss(\alpha,\xs) \geq 0$.
\end{proof}

As a special case, logical literals ($x$ or $\neg x$) constrain a single variable to take on a single value, and thus play a role similar to the labels used in supervised learning. Such constraints require an even tighter correspondence: semantic loss must act like a classical loss function (i.e., cross entropy).
\begin{axiom}[Label-Literal Correspondence] \label{ax:labelcorrespondence2}
  The semantic loss of a single literal is proportionate to the cross-entropy loss for the equivalent data label: $\sloss(x,p)  \propto -\log(p)$ and $\sloss(\neg x,p)  \propto -\log(1-p)$.
\end{axiom}

Next, we have the symmetry axioms.

\begin{axiom}[Value Symmetry] \label{ax:valsym}
  For all $\pv$ and $\alpha$, we have that $\sloss(\alpha,\pv) = \sloss(\bar{\alpha},1-\pv)$ where $\bar{\alpha}$ replaces every variable in $\alpha$ by its negation.
\end{axiom}

\begin{axiom}[Variable Symmetry] \label{ax:varsym}
  Let $\alpha$ be a sentence over $\Xs$ with probabilities $\pv$. Let $\pi$ be a permutation of the variables $\Xs$, let $\pi(\alpha)$ be the sentence obtained by replacing variables $x$ by $\pi(x)$, and let $\pi(\pv)$ be the corresponding permuted vector of probabilities. Then, $\sloss(\alpha,\pv) = \sloss(\pi(\alpha),\pi(\pv))$.
\end{axiom}

The value and variable symmetry axioms together imply the equality of the multiplicative constants in the label-literal duality axiom for all literals.
\begin{lemma}\label{l:consteq}
  There exists a single constant $K$ such that $\sloss(X,p)  = -K\log(p)$ and $\sloss(\neg X,p) = - K \log(1-p)$ for any literal $x$.
  \begin{proof}
   Value symmetry implies that $\sloss(X_i,\pv) = \sloss(\neg X_i,1-\pv)$. Using label-literal correspondence, this implies $K_1 \log(p_i)$ = $K_2 \log(1-(1-p_i))$ for the multiplicative constants $K_1$ and $K_2$ that are left unspecified by that axiom. This implies that the constants are identical. A similar argument based on variable symmetry proves equality between the multiplicative constants for different~$i$.
  \end{proof}
\end{lemma}

Finally, this allows us to prove the following form of semantic loss for a state $\xs$.
\begin{lemma}\label{l:slossstate2}
For state $\xs$ and vector $\pv$, we have $\sloss(\xs,\pv) \propto -\sum_{i: \xs \models \Xr_i} \log \pv_i - \sum_{i: \xs \models \neg \Xr_i} \log (1-\pv_i).$
\end{lemma}

\begin{proof}[Proof of Lemma~\ref{l:slossstate2}]
  A state $\xs$ is a conjunction of independent literals, and therefore subject to the additive independence axiom. Each literal's loss in this sum is defined by Lemma~\ref{l:consteq}.
\end{proof}

The following and final axiom requires that semantic loss is proportionate to the logarithm of a function that is additive for mutually exclusive sentences.
\begin{axiom}[Exponential Additivity] \label{lastaxiom}
 Let $\alpha$ and $\beta$ be mutually exclusive sentences (i.e., $\alpha \land \beta$ is unsatisfiable), and let $f^s(K,\alpha,\pv) = K^{-\sloss(\alpha,\pv)}$. Then, there exists a positive constant $K$ such that $f^s(K,\alpha \lor \beta,\pv) = f^s(K,\alpha,\pv) + f^s(K,\beta,\pv)$.
\end{axiom}

We are now able to state and prove the main uniqueness theorem.

\begin{theorem}[Uniqueness] \label{thm:unique2}
  The semantic loss function in Definition~\ref{def:sloss} satisfies
  all axioms in Appendix~\ref{a:axiomatization} and is the only function that does so, up to a multiplicative constant.
\end{theorem}

\begin{proof}[Proof of Theorem~\ref{thm:unique2}]
 The truth axiom states that $\forall \pv, f^s(K,\true,\pv) = 1$ for all positive constants~$K$. This is the first  Kolmogorov axiom of probability. The second Kolmogorov axiom for $f^s(K,.,\pv)$ follows from the additive independence axiom of semantic loss. The third Kolmogorov axiom (for the finite discrete case) is given by the exponential additivity axiom of semantic loss. Hence, $f^s(K,.,\pv)$ is a probability distribution for some choice of $K$, which implies the definition up to a multiplicative constant.
\end{proof}

\begin{figure*}[th]
\centering
  \begin{subfigure}[b]{0.2455\textwidth} \centering
    \includegraphics[width=0.50\linewidth]{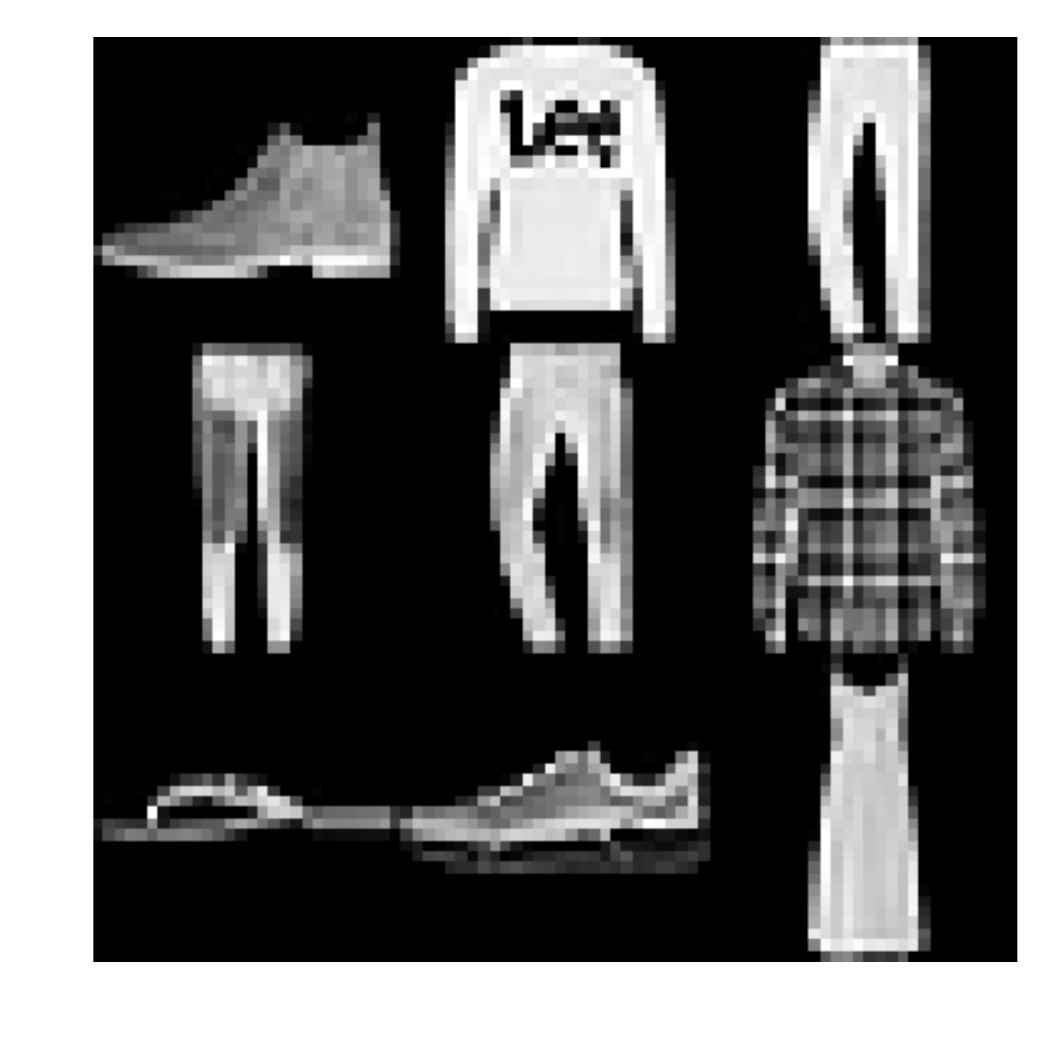}
    \caption{Confidently Correct} \label{figure: fashion:correct_confident}
  \end{subfigure}
  \begin{subfigure}[b]{0.2455\textwidth} \centering
    \includegraphics[width=0.50\linewidth]{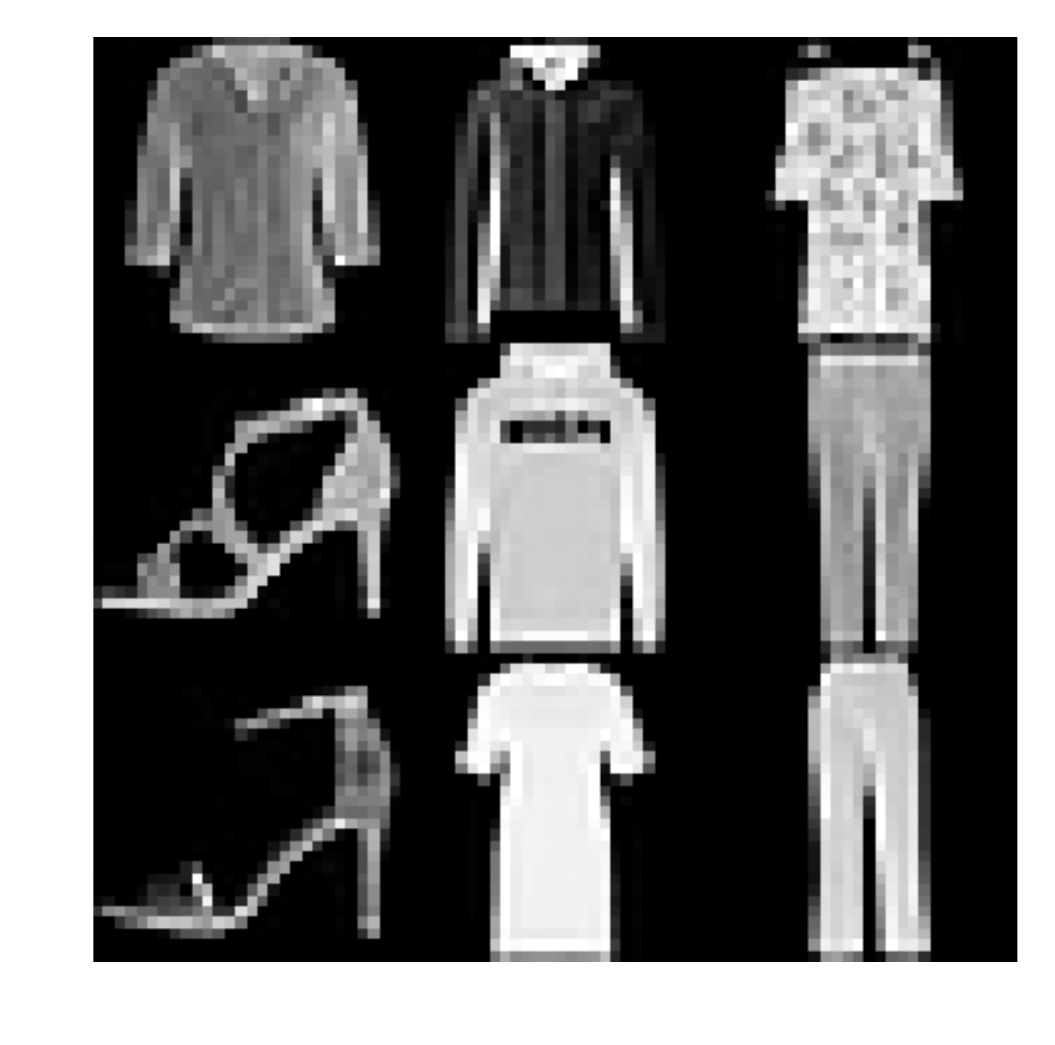}
    \caption{Unconfidently Correct} \label{figure: fashion:correct_unconfident}
  \end{subfigure}
  \begin{subfigure}[b]{0.2455\textwidth} \centering
    \includegraphics[width=0.50\linewidth]{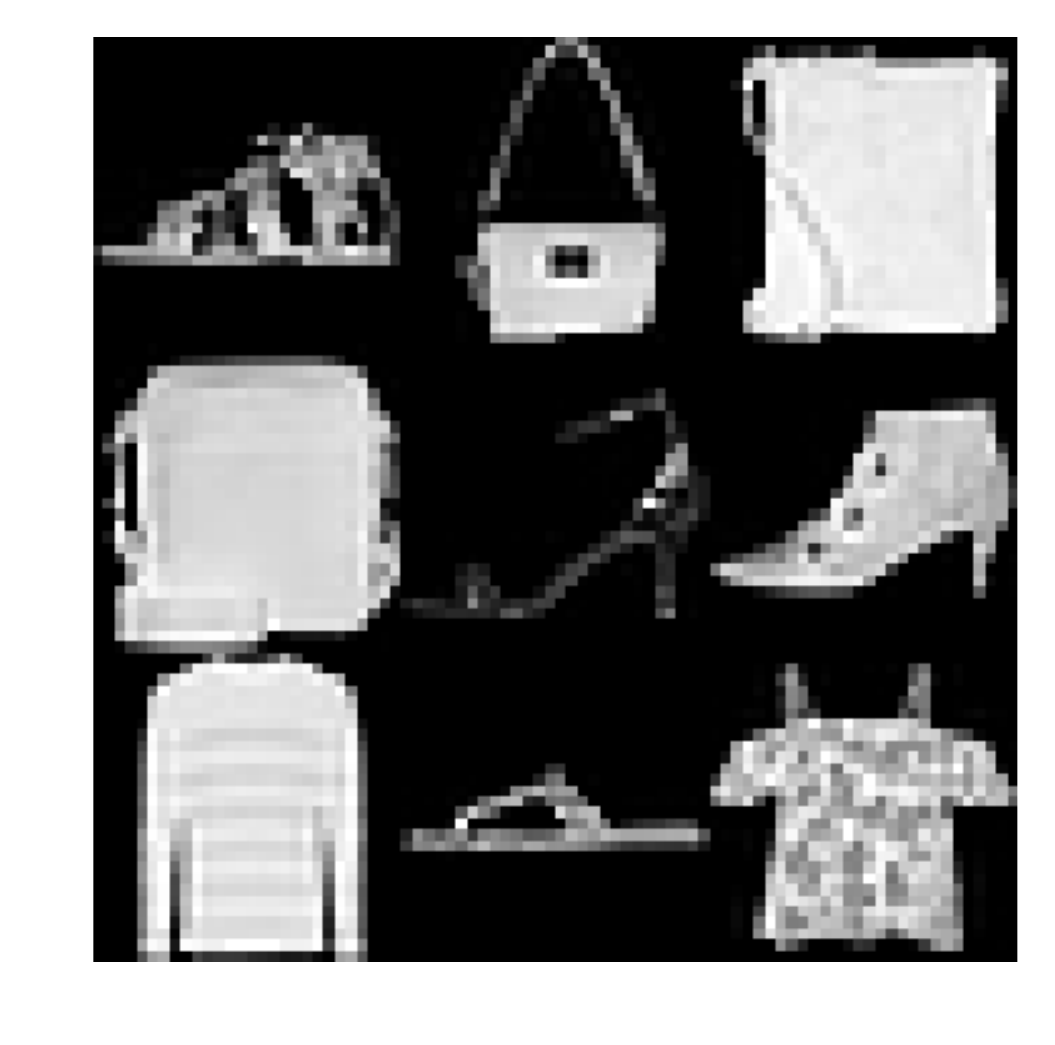}
    \caption{Unconfidently Incorrect} \label{figure: fashion:incorrect_unconfident}
  \end{subfigure}
  \begin{subfigure}[b]{0.2455\textwidth} \centering
    \includegraphics[width=0.50\linewidth]{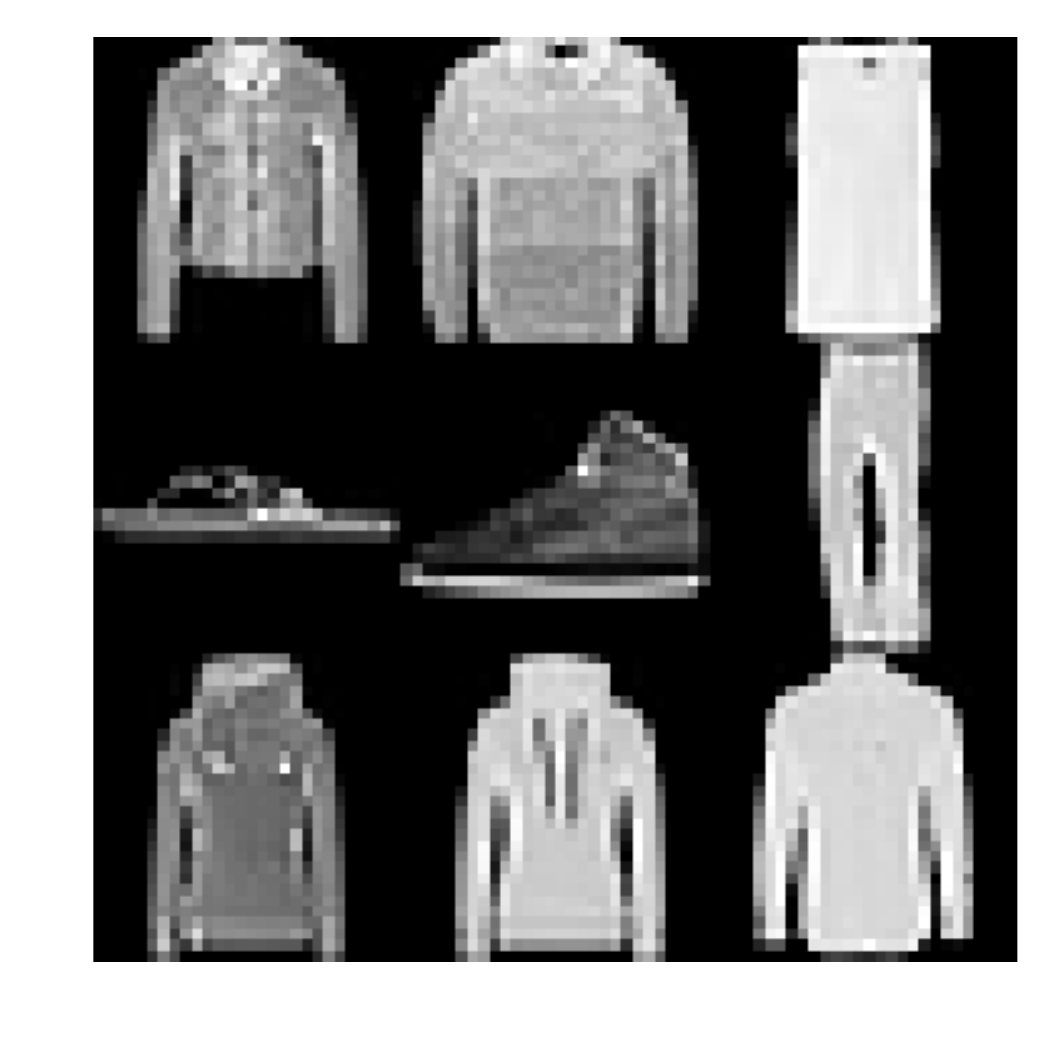}
    \caption{Confidently Incorrect} \label{figure: fashion:incorrect_confident}
  \end{subfigure}
   \caption{FASHION pictures grouped by how confidently the supervised base model classifies them correctly. With semantic loss, the final semi-supervised model predicts all correctly and confidently. }
   \label{figure: fashion}
\end{figure*}

\section{Specification of the Convolutional Neural Network Model} \label{appendix:specification}
Table~\ref{table:specification} shows the slight architectural difference between the CNN used in ladder nets and ours. The major difference lies in the choice of ReLu. Note we add standard padded cropping to preprocess images and an additional fully connected layer at the end of the model, neither is used in ladder nets. We only make those slight modification so that the baseline performance reported by  \citet{rasmus2015semi} can be reproduced.
\begin{table*}[htb]
\ra{1.02}
\centering \small
\caption {Specifications of CNNs in Ladder Net and our proposed method. }
\label{table:specification}
\begin{tabular}{  @{} c | c @{} }
CNN in Ladder Net & CNN in this paper\\
\midrule \midrule
\multicolumn{2}{c}{Input 32$\times$32 RGB image}\\
\midrule
& Resizing to $36\times36$ with padding \\
& Cropping Back \\
\midrule
\multicolumn{2}{c}{Whitening} \\
\multicolumn{2}{c}{Contrast Normalization}\\
\multicolumn{2}{c}{Gaussian Noise with std. of 0.3}\\
\midrule
3$\times$3 conv. 96 BN LeakyReLU  & 3$\times$3 conv. 96 BN ReLU\\
3$\times$3 conv. 96 BN LeakyReLU & 3$\times$3 conv. 96 BN ReLU\\
3$\times$3 conv. 96 BN LeakyReLU & 3$\times$3 conv. 96 BN ReLU\\
\midrule
\multicolumn{2}{c}{2$\times$2 max-pooling stride 2 BN}\\
\midrule
3$\times$3 conv. 192 BN LeakyReLU &3$\times$3 conv. 192 BN ReLU\\
3$\times$3 conv. 192 BN LeakyReLU & 3$\times$3 conv. 192 BN ReLU\\
3$\times$3 conv. 192 BN LeakyReLU & 3$\times$3 conv. 192 BN ReLU\\
\midrule
\multicolumn{2}{c}{2$\times$2 max-pooling stride 2 BN}\\
\midrule
3$\times$3 conv. 192 BN LeakyReLU & 3$\times$3 conv. 192 BN ReLU\\
1$\times$1 conv. 192 BN LeakyReLU & 3$\times$3 conv. 192 BN ReLU\\
1$\times$1 conv. 10~~~BN LeakyReLU & 1$\times$1 conv. 10~~~BN ReLU\\
\midrule
\multicolumn{2}{c}{Global meanpool BN}\\
\midrule
& Fully connected BN \\
\midrule
\multicolumn{2}{c}{10-way softmax}\\
\end{tabular}
\end{table*}

\section{Hyper-parameter Tuning Details} \label{appendix:tuning}
Validation sets are used for tuning the weight associated with semantic loss, the only hyper-parameter that causes noticeable difference in performance for our method. 
For our semi-supervised classification experiments, we perform a grid search over $\left\{0.001, 0.005, 0.01, 0.05, 0.1 \right\}$ to find the optimal value. Empirically, $0.005$ always gives the best or nearly the best results and we report its results on all experiments.

For the FASHION dataset specifically, because MNIST and FASHION share the same image size and structure, methods developed in MNIST should be able to directly perform on FASHION without heavy modifications. Because of this, we use the same hyper-parameters when evaluating our method. However, for the sake of fairness, we subject ladder nets to a small-scale parameter tuning in case its performance is more volatile.

For the grids experiment, the only hyper parameter that needed to be
tuned was again the weight given to semantic loss, which after trying $\left\{0.001,
  0.005, 0.01, 0.05, 0.1, 0.5, 1 \right\}$ was selected to be 0.5 based on
validation results.
For the preference learning experiment, we initially chose the semantic loss weight from $\left\{0.001,
  0.005, 0.01, 0.05, 0.1, 0.5, 1 \right\}$ to be 0.1 based on validation, and
then further tuned the weight to 0.25. 

\section{Specification of Complex Constraint Models} \label{a:complex}
\paragraph{Grids}
To compile our grid constraint, we first use \citet{NishinoYMN17a} to generate a
constraint for each source destination pair. Then, we conjoin each of these with
indicators specifying which source and destination pair must be used, and
finally we disjoin all of these together to form our constraint.

To generate the data, we begin by randomly removing one third of edges. We then
filter out connected components with fewer than 5 nodes to reduce degenerate
cases, and proceed with randomly selecting pairs of nodes to create data
points. 

The predictive model we employ as our baseline is a 5 layer MLP with 50 hidden
sigmoid units per layer. It is trained using Adam Optimizer, with full data batches
\citep{kingma2015adam}. Early stopping with respect to validation loss is used
as a regularizer.

\paragraph{Preference Learning}
We split each user's ordering into their ordering over sushis 1,2,3,5,7,8, which
we use as the features, and their ordering over 4,6,9,10 which are the labels we
predict. The constraint is compiled directly from logic, as this can be done in
a straightforward manner for an n-item ordering.

The predictive model we use here is a 3 layer MLP with 25 hidden sigmoid units
per layer. It is trained using Adam Optimizer with full data batches
\citep{kingma2015adam}. Early stopping with respect to validation loss is used
as a regularizer.

\section{Probabilistic Soft Logic Encodings} \label{a:psl}
We here give both encodings on the exactly-one constraint over three $x_1,x_2,x_3$. The first encoding is:
\[
(\neg x_1 \land x_2 \land x_3) \lor (x_1 \land \neg x_2 \land x_3) \lor (x_1 \land x_2 \land \neg x_3)
\]
The second encoding is:
\[
  (x_1 \lor x_2 \lor x_3) \land (\neg x_1 \lor \neg x_2) \land (\neg x_1 \lor \neg x_3) \land (\neg x_2 \lor \neg x_3)
\]
Both encodings extend to cases whether the number of variables is arbitrary.

The norm functions used for these experiments are as described in \citet{Kimmig2012ASI}, with the loss for an interpretation $I$ being defined as follows:
\begin{align*}
    x_1 \land x_2 &= max\{0, I(x_1) + I(x_2) - 1\} \\
    x_1 \lor x_2 &= min\{ I(x_1) + I(x_2), 1\} \\
    \neg x_1 &= 1-I(x_1)
\end{align*}
\end{document}